\documentclass[10pt,a4paper]{article}

\usepackage[margin=1in]{geometry}
\usepackage{microtype}
\usepackage{booktabs}
\usepackage{graphicx}
\usepackage{bm}

\usepackage{amsmath,amssymb,amsthm,mathtools}
\newtheorem{theorem}{Theorem}
\newtheorem{proposition}{Proposition}
\newtheorem{corollary}{Corollary}

\theoremstyle{definition}
\newtheorem{assumption}{Assumption}
\newtheorem{definition}{Definition}
\newtheorem{example}{Example}
\newtheorem{remark}{Remark}

\usepackage[colorlinks=true,linkcolor=blue,citecolor=blue,urlcolor=blue]{hyperref}

\newcommand{\A}{\mathcal{A}}
\newcommand{\X}{\mathcal{X}}
\newcommand{\U}{\mathcal{U}}
\newcommand{\Gc}{\mathcal{G}^{*}}
\newcommand{\Ah}{\widehat{A}}
\newcommand{\Av}{A_{\mathrm{viable}}}
\newcommand{\Ac}{A_{\mathrm{common}}}
\newcommand{\Hf}{\mathcal{H}}
\newcommand{\Hm}{\mathcal{H}^{-}}
\newcommand{\Ff}{\mathcal{F}}
\newcommand{\Prob}{\mathbb{P}}
\newcommand{\E}{\mathbb{E}}
\newcommand{\tv}{d_{\mathrm{TV}}}
\newcommand{\abar}{\bar{\alpha}}
\newcommand{\Ccert}{C_{\mathrm{cert}}}
\newcommand{\esssup}{\operatorname*{ess\,sup}}

\usepackage{authblk}

\title{\LARGE\bf Generator-Independent Runtime Assurance\\ under Partial Observation}

\author[1]{Guangxi Wan\thanks{Draft v0.5.}}
\author[1,2]{Yongbo Xie}
\author[1]{Yuqi Liu}
\author[1]{Qingwei Dong}
\author[1]{Qingxin Li}
\author[1]{Hongfei Bai}
\author[1,*]{Peng Zeng}

\affil[1]{State Key Laboratory of Robotics and Intelligent Systems, Shenyang Institute of Automation, Chinese Academy of Sciences, Shenyang 110016, China}
\affil[2]{University of Chinese Academy of Sciences, Beijing 100049, China}
\affil[*]{Corresponding author: \texttt{zp@sia.cn}}

\date{}

\begin{document}

\maketitle

\begin{abstract}
Proposal-based controllers---learned policies, language-model planners, and other black-box \emph{generators}---are increasingly deployed behind runtime verification gates. We ask when the closed-loop safety guarantee decouples from the generator. The prevailing per-candidate certification pattern does not compose: under retry or best-of-$k$ selection a per-candidate false-admission level $\alpha$ can inflate to $1-(1-\alpha)^{k}$. Our main theorem shows that \emph{simultaneous setwise soundness}---certifying a set of admissible proposals containing no nonviable action---is necessary and sufficient for generator-independent \emph{admission soundness}, the worst case over all generators of executing a nonviable proposal equalling the probability of setwise failure; together with a design-time certificate and a no-bypass rule it is sufficient for \emph{contract safety}, with violation bound $\Gamma+\sum_t\varepsilon_t+\eta$ invariant under arbitrary, even adversarial, replacement of the generator. A second theorem bounds every admission mechanism under partial observation: for a fixed probing and admission policy, if two state hypotheses whose information laws lie within total-variation distance $\delta$ require different safe decisions, then $\abar+\beta+\delta\ge1$. A sequential risk ledger makes the guarantee implementable with time-uniform confidence tubes, and shows that deterministic admission computations concentrate all statistical risk in state estimation. Simplex-style runtime assurance and control-barrier-function filtering are recovered as degenerate cases.
\\[0.8em]
\end{abstract}

\vspace{1.5em}

\section{Introduction}\label{sec:intro}

An increasingly common control architecture places an unverified, high-capability \emph{generator}---a reinforcement-learned policy, a language-model planner, a sampling-based optimizer---behind a runtime verification gate. The gate inspects proposed actions and admits, modifies, or rejects them; rejections transfer control to a certified fallback. One canonical lineage is the Simplex architecture \cite{seto1998simplex,sha2001simplicity}; the pattern is central to modern runtime assurance (RTA) for learning-enabled systems \cite{hobbs2023rta,dunlap2023asif}, has close analogues in control-barrier-function (CBF) safety filtering \cite{ames2019cbf,gurriet2018asif}, and appears as shielding in reinforcement learning \cite{alshiekh2018shielding}.

The promise of the architecture is \emph{separation}: the safety case should attach to the gate, not to the generator, so that the generator may be retrained, updated, or replaced without repeating the safety argument. This promise is what makes the architecture attractive for rapidly evolving learned components, whose version cadence is incompatible with re-certification of the full loop \cite{dunlap2023asif}. Yet the conditions under which the promise holds have not, to our knowledge, been stated precisely, and a natural and widely used certification pattern violates them.

\subsection{A motivating failure: per-candidate certification does not compose}\label{sec:motivating}

Consider a verifier that subjects each proposed action to a statistical test calibrated to a false-admission level $\alpha$: for every candidate $a$, the probability of admitting $a$ when $a$ is nonviable at the true state is at most $\alpha$. This per-candidate pattern appears whenever admission is decided by evaluating a learned critic, a Monte-Carlo rollout, or a hypothesis test on the proposed action alone.

\begin{example}[Selection inflation]\label{ex:selection}
Let the state be $x\in\mathbb{R}^n$, observed through $y=x+\epsilon$ with $\epsilon\sim\mathcal N(0,I_n)$. Candidate actions $a_1,\dots,a_n$ are available; action $a_k$ is nonviable at $x$ iff $x_k\ge 0$. Suppose the true state has $x_k=\Delta>0$ for all $k$, so every candidate is nonviable with margin $\Delta$. The verifier tests candidate $a_k$ using only the $k$-th coordinate: it admits $a_k$ iff $y_k\le c$, where $\Phi(c-\Delta)=\alpha$. Each candidate individually enjoys false-admission level exactly $\alpha$. Now let the generator submit $a_1,a_2,\dots$ in turn and execute the first admitted candidate---the standard operating mode of best-of-$k$ sampling and of any planner that retries after rejection. The probability that some nonviable candidate is admitted is
\[
1-(1-\alpha)^{n}\ \longrightarrow\ 1\qquad(n\to\infty),
\]
e.g.\ $0.40$ for $n=10$ and $0.92$ for $n=50$ at $\alpha=0.05$.
\end{example}

No adversarial intent is required: candidate diversity and retry alone are sufficient to induce the selection effect, and retry is precisely how modern generators are operated. The inflation in Example~\ref{ex:selection} assumes independent test statistics; without any dependence assumption the worst case over joint distributions is $\min(1,k\alpha)$ (attained when the per-candidate failure events are disjoint), while full dependence yields no inflation. The correct summary is therefore not that retry always inflates, but that \emph{an unadjusted per-candidate level does not remain a generator-level level under unrestricted best-of-$k$ selection}. The phenomenon is analogous to post-selection effects in statistical inference \cite{taylor2015selective}: the quantity that was calibrated is not the quantity that the closed loop realizes once selection is allowed.

\subsection{Contributions}

\emph{(i) Separation.} Theorem~\ref{thm:separation} identifies the exact object supporting generator-independent guarantees: \emph{simultaneous setwise soundness}, i.e.\ the verifier outputs a set $\Ah$ of admissible proposals and the controlled event is $\{\Ah\subseteq\Av(x)\}$, that the certified set contains no nonviable action. The two halves have deliberately different scopes, and we state them separately throughout: setwise soundness is \emph{necessary and sufficient} for generator-independent admission soundness, and, together with the design-time certificate, \emph{sufficient} for generator-independent contract safety (Remark~\ref{rem:scope}). Part (a) is an exact one-step characterization: with the past strategy fixed, the essential supremum over stage selectors---including selectors with access to the realized state and all noise variables realized up to the decision time---of the conditional probability of executing a nonviable \emph{proposal} equals the conditional probability of setwise failure. Part (b) composes the characterization over a horizon: under a design-time certificate and a no-bypass rule, the closed-loop violation probability is bounded by $\Gamma+\sum_t\varepsilon_t+\eta$ uniformly over all generators, and the bound is invariant under mid-run generator replacement. Safety attaches to the gate; the generator affects performance (abstention frequency, task success) only.

\emph{(ii) Impossibility under partial observation.} Theorem~\ref{thm:impossibility} is a single-experiment converse: for a fixed probing/admission policy $\pi$, if two state hypotheses with information laws $\tv(P^{\pi}_{x},P^{\pi}_{x'})\le\delta$ require different safe decisions, then $\abar+\beta+\delta\ge1$, where $\abar$ is the statewise setwise risk and $1-\beta$ the required usefulness. For observationally equivalent hypotheses ($\delta=0$, all policies) reliability and usefulness cannot simultaneously be high, at any observation budget. The result requires no linear or Gaussian observation model.

\emph{(iii) Implementation.} Proposition~\ref{prop:ledger} supplies the sequential risk ledger in two parallel forms---a per-step conditional form and a horizon-uniform form requiring no conditional-coverage claim---together with the design principle that deterministic admission computations force $\varepsilon_t\equiv0$, concentrating all statistical risk in state estimation.

Simplex-style RTA and CBF filtering are recovered as degenerate cases (Section~\ref{sec:consequences}). A companion paper \cite{companion} develops the quantitative information geometry of the admission set for the local linear-Gaussian setting; the present paper is deliberately free of that structure.

\subsection{Related work}\label{sec:related}

\emph{Runtime assurance.} The Simplex architecture \cite{seto1998simplex,sha2001simplicity} introduced the unverified-controller/certified-fallback/switching pattern; \cite{hobbs2023rta} surveys modern RTA and \cite{dunlap2023asif} connects active set-invariance filtering to airworthiness certification. Canonical formulations typically analyze a single primary-controller output at each decision time; the selection phenomenon of Section~\ref{sec:motivating}, and hence the exactness direction of Theorem~\ref{thm:separation}(a), does not arise there. Our Part (b) may be read as the RTA guarantee re-proved at the level of generality at which it survives retrying, post-selecting, state-aware generators.

\emph{Safety filters.} CBF-based filters \cite{ames2019cbf,gurriet2018asif} implicitly output an admissible \emph{set} (the constraint set of a quadratic program) and are therefore selection-robust in the sense of Theorem~\ref{thm:separation}(a) under full-state, correct-model conditions. Under partial observation the guarantee requires measurement-robust variants \cite{lindemann2021robustocbf}; our Theorem~\ref{thm:impossibility} bounds what any such variant can achieve.

\emph{Output-feedback and belief-space safety.} Set-membership state estimation goes back to \cite{bertsekas1971set}, and safety under estimation error via error tubes and robust output CBFs is developed in \cite{lindemann2021robustocbf,rober2026guardian,vahs2026bcbf}. Belief-space safety filters \cite{vahs2023riskaware,hu2026beliefsf} establish probabilistic safety under state uncertainty and observe that better inference permits less conservative filtering. Relative to this line our contributions are the converse direction and the adversary class: an information-theoretic \emph{lower} bound ($\abar+\beta+\delta\ge1$) rather than a conservative construction, and guarantees uniform over a generator class $\Gc$ that may post-select---a requirement strictly stronger than validity under the trajectory distribution of one nominal closed loop, which is the setting of conformal-style belief-space certificates \cite{hu2026beliefsf}.

\emph{Separation in supervisory control.} The term \emph{separation} has an established and different meaning in control: the estimator may be designed without reference to the control law, the state estimate serving as a sufficient statistic \cite{wonham1968sep,witsenhausen1971sep}. Barrett and Lafortune \cite{barrett2000separation} give the logical discrete-event analogue. The maximal information set of a centralized supervisor---the closed-loop traces indistinguishable from the realized one---can be recovered from the observations together with the realized control \emph{actions}, with no knowledge of the control \emph{policy}, and admits a finite-state sufficient statistic when the behaviors are regular; the property generally fails in the decentralized setting, where one controller's information set depends on another's contingency plans for sequences it neither observes nor disables. We separate a different pair: not the estimator from the control law, but the safety guarantee from an upstream generator, which is not a party to that model. Their result is nevertheless the ancestor of one mechanism used here. The verifier-observable history $\Hm_t$ contains executed actions but never the generator's policy (Section~\ref{sec:info}), so the tube and the certified set are policy-free for their reason; what Theorem~\ref{thm:separation} adds is that the \emph{risk level} carried by the certified set is policy-free as well, and that this holds exactly when the certified set is setwise sound.

\emph{Observability under partial observation.} Whether a specification is achievable at all under a projected observation channel is governed by observability conditions on indistinguishable strings: strings the supervisor cannot tell apart may not demand conflicting control decisions \cite{linwonham1988,cieslak1988}. The $\delta=0$ corner of Theorem~\ref{thm:impossibility} is the statistical form of that condition, and we claim no priority over the qualitative statement. What the present form adds is quantitative and model-free: error levels in place of exact achievability, a price $\delta$ for partial distinguishability, and no automaton, language, or projection structure.

\emph{Selective inference.} The inflation of Example~\ref{ex:selection} parallels post-selection effects in statistics \cite{taylor2015selective,benjamini1995fdr}; the object here is action admission in a closed loop rather than inference, and the repair (setwise soundness of a certified set) plays the role that simultaneous confidence statements play there.

\section{Problem Formulation}\label{sec:formulation}

All spaces are Polish; measurability requirements beyond this are collected in Assumption~\ref{ass:selection}.

\subsection{Plant, contract, and viability}

The state $x_t\in\X$ evolves as a controlled Markov process $x_{t+1}\sim f(\cdot\mid x_t,a_t)$ over a horizon $t=1,\dots,T$, where $a_t$ is either a \emph{proposal action} in $\A$ or an engagement of a separately certified \emph{fallback mode}; the fallback is a recovery policy outside the proposal space $\A$ (this modeling choice is discussed in Remark~\ref{rem:fallbackmode}). A \emph{contract} is a measurable $\varphi:\X\to\{0,1\}$ ($1=$ violation); the trajectory violates the contract if $\varphi(x_t)=1$ for some $t\le T$.

The interface between contract and admission is a set-valued \emph{viability map} $\Av:\X\rightrightarrows\A$: $\Av(x)$ collects the proposal actions whose execution at $x$ preserves the contract and the feasibility of the fallback (the standard one-step-return condition of RTA \cite{hobbs2023rta}). How $\Av$ and the associated design-time certificates are computed is not the subject of this paper.

\subsection{Information structure}\label{sec:info}

Three $\sigma$-algebras must be distinguished. $\Hm_t$ denotes the \emph{verifier-observable} information available immediately before the $t$-th admission computation is run: past excitations and all observations available before the decision, past executed actions and fallback engagements, past certified sets, and past verifier randomization. The admission computation may draw fresh randomization $\zeta_t$, independent of the plant given $\Hm_t$, and we write $\Hf_t:=\Hm_t\vee\sigma(\zeta_t)$ for the information after the computation has run. $\Ff_t\supseteq\Hf_t$ denotes the \emph{full physical} history, additionally containing the latent states $x_{0:t}$ and all noise variables realized up to the decision time. In general $x_t$ is $\Ff_t$-measurable but not $\Hf_t$-measurable; every statistical statement about the verifier below is conditioned on $\Hm_t$ or $\Hf_t$, never on $\Ff_t$.

The split at the verifier's output is not cosmetic. Risk levels charged to the verifier are properties of the computation \emph{before} its coins are tossed and are therefore conditioned on $\Hm_t$; the exactness statement of Theorem~\ref{thm:separation}(a) concerns the certified set \emph{as realized} and is therefore conditioned on $\Hf_t$. Conditioning a statement about the realized output on $\Hf_t$ would make it degenerate, since any $\Hf_t$-measurable event has conditional probability $0$ or $1$ there.

A \emph{verifier} outputs, before each decision, a certified set $\Ah_t\subseteq\A$ that is $\Hf_t$-measurable; $\Ah_t=\emptyset$ encodes abstention. The verifier may also choose probing excitations adaptively; these choices are $\Hm$-adapted and are part of the observable history. A verifier whose admission computation uses no fresh randomization is called \emph{deterministic}; for such a verifier $\Ah_t$ is $\Hm_t$-measurable.

\subsection{Generators and the execution rule}

A \emph{generator} is a sequence of stage maps
\[
G_t:\ \bigl(\Ff_t\vee\sigma(\Ah_t)\bigr)\text{-measurable, randomized}\ \longmapsto\ g_t\in\A ,
\]
with internal randomness independent of the verifier's. We write $\Gc$ for the class of all such sequences. The class is deliberately extreme: the stage map may inspect the realized latent state, all noise variables realized up to the decision time, and the certified set itself, and may coordinate across time. $\Gc$ subsumes retrying, best-of-$k$ selection, distribution shift, version drift, and deliberate adversaries; any guarantee uniform over $\Gc$ is a fortiori uniform over realistic generators.

\begin{assumption}[No bypass]\label{ass:nobypass}
The executed action is
\[
a_t=\begin{cases} g_t, & g_t\in\Ah_t,\\ \text{fallback engagement}, & \text{otherwise,}\end{cases}
\]
the generator has no other channel to the actuator, and $\Ah_t$ is $\Hf_t$-measurable---in particular, it does not depend on generator-supplied content beyond what is already in $\Hm_t$.
\end{assumption}

\subsection{Certificate and regularity}

Let $B_t:=\{\Ah_t\not\subseteq\Av(x_t)\}$ denote \emph{setwise failure} at time $t$, and $E_{t-1}:=\bigcap_{s<t}B_s^{\,c}$ the \emph{certified prefix} event ($E_0$ is the whole space). Call an execution \emph{admissible up to $T$} if for every $t\le T$, either $a_t\in\Av(x_t)$ (a viable proposal is executed) or the fallback mode is engaged at a time $t$ at which $E_{t-1}$ holds.

\begin{assumption}[Design-time certificate]\label{ass:certificate}
There is an event $\Ccert$ with
\[
\sup_{G\in\Gc}\Prob_G\bigl(\Ccert^{\,c}\bigr)\ \le\ \eta ,
\]
such that, pathwise on $\Ccert$, every admissible execution satisfies $\varphi(x_t)=0$ for all $t\le T$.
\end{assumption}

Assumption~\ref{ass:certificate} packages the design-time content of RTA---fallback invariance from certified prefixes, one-step return, and the residual risk of the stochastic invariance argument---into a single certificate-valid event whose failure probability is bounded uniformly over generators. The pathwise clause is exactly what a design-time safety case certifies; $\eta$ is its residual risk.

\begin{assumption}[Selection regularity]\label{ass:selection}
For every $t$: (i) the correspondence $D_t:=\Ah_t\setminus\Av(x_t)$ has $\{D_t\neq\emptyset\}=B_t$ measurable and admits a measurable selector on $B_t$; (ii) the events $\{x_t\in C_t\}$ and $\{\Ah_t\subseteq \Ac(C_t)\}$ appearing in Section~\ref{sec:ledger} are measurable with respect to the indicated $\sigma$-algebras, for the tube processes $C_t$ used there.
\end{assumption}

Assumption~\ref{ass:selection} is trivial when $\A$ is finite or countable. For general Polish action spaces we impose it explicitly: sufficient measurable-selection conditions are standard, but they must be verified for the induced \emph{difference} correspondence $D_t$, whose regularity is not inherited from that of $\Ah_t$ and $\Av$ (a difference of closed-valued correspondences need not be closed-valued). We adopt the assumption directly so that the main theorems carry exactly the hypotheses they use.

\begin{definition}[Uniform risk levels]\label{def:levels}
Deterministic scalars $\abar_1,\dots,\abar_T\in[0,1]$ are \emph{one-step setwise risk bounds} for the verifier if
\[
\sup_{G\in\Gc}\ \esssup\ \Prob_G\bigl(B_t\,\big|\,\Hm_t\bigr)\ \le\ \abar_t ,\qquad t=1,\dots,T .
\]
\end{definition}

The essential supremum is over histories; the supremum over $\Gc$ makes the bounds uniform over the history distributions that generators can induce.

\section{Generator--Guarantee Separation}\label{sec:separation}

\begin{theorem}[Generator--guarantee separation]\label{thm:separation}
\emph{(a) One-step exact characterization.} Let Assumptions~\ref{ass:nobypass} and~\ref{ass:selection} hold, fix $t$, and fix any strategy for stages $1,\dots,t-1$. Then, for every stage-$t$ selector, pathwise,
\[
\{a_t\in\A\ \text{and}\ a_t\notin\Av(x_t)\}\ \subseteq\ B_t ,
\]
and there exists a stage-$t$ selector $G_t^{\star}$ (composable with the fixed past strategy) for which the two events coincide. Consequently,
\begin{equation}\label{eq:onestep}
\esssup_{G_t}\ \Prob\bigl(a_t\in\A,\ a_t\notin\Av(x_t)\,\big|\,\Hf_t\bigr)
\;=\;
\Prob\bigl(B_t\,\big|\,\Hf_t\bigr)\quad\text{a.s.},
\end{equation}
where the essential supremum is over stage-$t$ selectors with the past strategy fixed. In particular, a level-$\abar_t$ generator-independent bound on the executed-nonviable-proposal probability holds if and only if the verifier is setwise sound at level $\abar_t$.

\emph{(b) Horizon composition.} Let the no-bypass, design-time-certificate and selection-regularity assumptions (Assumptions~\ref{ass:nobypass}, \ref{ass:certificate} and~\ref{ass:selection}) hold, and let $\abar_1,\dots,\abar_T$ be one-step setwise risk bounds. Then
\begin{equation}\label{eq:horizon}
\sup_{G\in\Gc}\ \Prob_G\bigl(\exists\,t\le T:\ \varphi(x_t)=1\bigr)\ \le\ \sum_{t=1}^{T}\abar_t+\eta ,
\end{equation}
and the bound is unchanged if the generator is replaced arbitrarily---including adversarially and mid-run---by any other element of $\Gc$. The generator affects only performance quantities (abstention frequency, task-level success), never the bound \eqref{eq:horizon}.
\end{theorem}

\begin{proof}
(a) The first inclusion is immediate: an executed proposal lies in $\Ah_t$ by Assumption~\ref{ass:nobypass}, so if it is nonviable then $\Ah_t\not\subseteq\Av(x_t)$. For the converse, let $g^{\star}$ be a measurable selector of $D_t$ on $B_t$ (Assumption~\ref{ass:selection}) and let $G_t^{\star}$ play $g^{\star}$ on $B_t$ (the stage map observes $x_t$ and $\Ah_t$, so $B_t$ is in its information) and a fixed default $a_0\in\A$ on $B_t^{\,c}$. On $B_t^{\,c}$ no nonviable proposal is executed: either $a_0\in\Ah_t\subseteq\Av(x_t)$, or $a_0\notin\Ah_t$ and the fallback mode is engaged, which is not a proposal. (A default is used rather than ``some admitted action'' because $\Ah_t$ may be empty.) Hence the executed action is a nonviable proposal exactly on $B_t$, and taking conditional probabilities given $\Hf_t$ under the fixed past-strategy law yields \eqref{eq:onestep}.

(b) Fix any $G\in\Gc$ and consider the event $\Ccert\cap\bigcap_{t\le T}B_t^{\,c}$. On this event, at each $t$ either $a_t=g_t\in\Ah_t\subseteq\Av(x_t)$, or the fallback is engaged at a time at which $E_{t-1}\supseteq\bigcap_{s\le T}B_s^{\,c}$ holds; hence the execution is admissible up to $T$, and by the pathwise clause of Assumption~\ref{ass:certificate} the contract holds throughout. Therefore, pathwise and for every $G$,
\[
\{\text{violation}\}\ \subseteq\ \Ccert^{\,c}\ \cup\ \bigcup_{t\le T}B_t ,
\]
and
\[
\Prob_G(\text{violation})\le\eta+\sum_{t}\E_G\bigl[\Prob_G(B_t\mid\Hm_t)\bigr]\le\eta+\sum_t\abar_t ,
\]
where the last step uses that the bounds of Definition~\ref{def:levels} are uniform over $\Gc$ and over histories; no step integrates over the law of $G$. Mid-run replacement changes the induced history distribution but not the uniform conditional bounds, hence not \eqref{eq:horizon}.
\end{proof}

\begin{remark}[What is exact, and what is only sufficient]\label{rem:scope}
The two halves of Theorem~\ref{thm:separation} have different logical strengths and we keep them apart. Part (a) is an equivalence, but about \emph{admission soundness}: simultaneous setwise soundness is necessary and sufficient for a generator-independent bound on the probability of executing a nonviable proposal. Part (b) is an implication: setwise soundness \emph{together with} the design-time certificate is sufficient for a generator-independent bound on contract violation. The converse of (b) does not hold and is not claimed---an action outside $\Av(x)$ merely lacks a viability certificate, which need not entail an actual contract violation, and a system may also satisfy its contract for reasons the certificate does not capture. Accordingly we never describe setwise soundness as necessary for safety.
\end{remark}

\begin{remark}[Why part (a) needs no prefix restriction]\label{rem:prefix}
The one-step statement concerns executed nonviable \emph{proposals}. Fallback engagements are not proposals: their safety is a design-time matter, accounted once---together with its certified-prefix scope---inside the pathwise clause of Assumption~\ref{ass:certificate}, whose scope is respected in the proof of (b) because on $\bigcap_s B_s^{\,c}$ every fallback engagement occurs with the prefix intact. Separating the two channels in this way is what makes \eqref{eq:onestep} an unconditional identity: the generator's only channel for injecting harm is the proposal channel, and on that channel the worst case over all stage selectors is exactly setwise failure.
\end{remark}

\begin{remark}[Fallback as a certified mode]\label{rem:fallbackmode}
Modeling the fallback as a mode outside $\A$, rather than as a distinguished element of $\A$, avoids a degeneracy in Section~\ref{sec:consequences}: if the fallback were a proposal action certified viable throughout the uncertainty set, the common admissible set could never be empty and abstention would be vacuous. As a separate mode, ``no proposal is certifiable'' and ``transfer to the certified recovery mode'' are distinct, meaningful outputs.
\end{remark}

\subsection{Why each safeguard matters}\label{sec:safeguards}

Dropping any of the following safeguards \emph{without a replacement condition} can destroy generator-independent assurance; we record one failure mode per safeguard. We do not claim the assumption set is minimal in a formal sense.

\begin{remark}[Per-candidate calibration]\label{rem:percandidate}
Example~\ref{ex:selection} exhibits a verifier whose per-candidate false-admission level is $\alpha$ for every candidate, yet whose induced set $\Ah=\{a_k:y_k\le c\}$ has setwise failure probability $1-(1-\alpha)^n$. Viewed through \eqref{eq:onestep}, the theorem does not fail; per-candidate calibration simply measures the wrong event.
\end{remark}

\begin{remark}[Generator-supplied content]\label{rem:metadata}
Dependence of $\Ah_t$ on unverified generator content is harmless precisely when setwise soundness remains uniform over that content. A harmless example: $\Ah_t=A_{\mathrm{certified}}\cap A_{\mathrm{suggested}}$, which intersects a certified set with generator suggestions and can only shrink the certified set. A harmful example: a shortcut rule that waives the certification when the generator reports high confidence---an adversarial generator reporting maximal confidence converts the waiver into a bypass, and no level $\abar<1$ is sustainable. The operative requirement is uniformity of the soundness guarantee, not blanket independence.
\end{remark}

\begin{remark}[Bypass]\label{rem:bypass}
If with probability $p$ the generator can write to the actuator directly, the violation probability is bounded below by $p$ times the generator's nonviable-play rate, which is unconstrained over $\Gc$.
\end{remark}

\begin{remark}[Guarantee--performance split]\label{rem:split}
Theorem~\ref{thm:separation} licenses exactly the practice it was designed to license: the generator may be retrained, updated, or swapped without touching the safety case, because \eqref{eq:horizon} never referenced it. What the generator does determine is how often its proposals fall inside $\Ah_t$---the price of a poor generator is abstention and fallback engagement, not safety.
\end{remark}

\section{Impossibility under Partial Observation}\label{sec:impossibility}

Theorem~\ref{thm:separation} presupposes a verifier achieving small setwise risk. Whether such a verifier exists is an informational question about the observation process. We state the obstruction as a single-decision statistical experiment; the sequential setting of Section~\ref{sec:formulation} contains one such experiment at each decision epoch.

Fix a probing/admission policy $\pi$: a rule that selects excitations, collects observations, and outputs a certified set, using internal randomization as needed. For a state hypothesis $x$ held fixed during the decision, let $P^{\pi}_{x}$ denote the joint law of all random elements available to the admission decision (observations and randomization), and let $\Ah^{\pi}$ denote the output.

\begin{definition}[Statewise levels]\label{def:statewise}
The policy $\pi$ has \emph{statewise setwise risk} $\abar$ at $x'$ if $P^{\pi}_{x'}\bigl(\Ah^{\pi}\not\subseteq\Av(x')\bigr)\le\abar$, and \emph{usefulness} $1-\beta$ at $(x,a)$ if $P^{\pi}_{x}\bigl(a\in\Ah^{\pi}\bigr)\ge1-\beta$.
\end{definition}

\begin{theorem}[Admission incompatibility under observational ambiguity]\label{thm:impossibility}
Let $a\in\Av(x)\setminus\Av(x')$. If the policy $\pi$ has statewise setwise risk $\abar$ at $x'$ and usefulness $1-\beta$ at $(x,a)$, then
\begin{equation}\label{eq:impossibility}
\abar+\beta+\tv\bigl(P^{\pi}_{x},P^{\pi}_{x'}\bigr)\ \ge\ 1 .
\end{equation}
\end{theorem}

\begin{proof}
Since $a\notin\Av(x')$, the event $\{a\in\Ah^{\pi}\}$ is contained in the setwise-failure event at $x'$, so $P^{\pi}_{x'}(a\in\Ah^{\pi})\le\abar$. The indicator of $\{a\in\Ah^{\pi}\}$ is a randomized functional of the elements whose law is $P^{\pi}_{\cdot}$, so $P^{\pi}_{x}(a\in\Ah^{\pi})\le P^{\pi}_{x'}(a\in\Ah^{\pi})+\tv(P^{\pi}_{x},P^{\pi}_{x'})$. Combining with usefulness, $1-\beta\le\abar+\tv(P^{\pi}_{x},P^{\pi}_{x'})$.
\end{proof}

\begin{corollary}[Uniform version over policies]\label{cor:deltaT}
Let $\delta_T(x,x'):=\sup_{\pi}\tv(P^{\pi}_{x},P^{\pi}_{x'})$, the supremum over admissible policies at observation budget $T$ (adaptive probing included). If $\delta_T(x,x')\le\delta$, then every admissible policy $\pi$, with its own levels $(\abar_\pi,\beta_\pi)$, satisfies $\abar_\pi+\beta_\pi+\delta\ge1$. In particular, for observationally equivalent hypotheses ($\delta_T=0$), $\abar_\pi+\beta_\pi\ge1$ for every $\pi$: reliability and usefulness cannot simultaneously be high, at any budget, for any policy.
\end{corollary}

The supremum over policies makes $\delta_T$ adversarial in the verifier's favor: it is the distinguishability achievable by the \emph{best} probing strategy. For a linear-Gaussian observation model $y_s=J(u_s)x+\epsilon_s$, $\epsilon_s\sim\mathcal N(0,\Sigma)$ i.i.d., the chain rule for Kullback--Leibler divergence gives, for every adaptive policy, $\mathrm{KL}\le \tfrac{T}{2}\sup_{u}\|\Sigma^{-1/2}J(u)(x-x')\|^2$, whence by Pinsker
\begin{equation}\label{eq:gaussdelta}
\delta_T(x,x')\ \le\ \tfrac{\sqrt{T}}{2}\,\sup_{u\in\U}\bigl\|\Sigma^{-1/2}J(u)(x-x')\bigr\| ,
\end{equation}
and $\delta_T\equiv0$ at every budget whenever $x-x'$ lies in the common null space of the whitened observation maps: adaptivity earns at most the $\sqrt{T}$ factor and cannot manufacture distinguishability that no excitation possesses. We emphasize that \eqref{eq:gaussdelta} is an instance; Theorem~\ref{thm:impossibility} requires no linear or Gaussian observation model.

Three readings of \eqref{eq:impossibility} are worth separating. First, it is an impossibility result for \emph{mechanisms}, not for estimators: no admission rule escapes it, because the proof uses only the proximity of two information laws and the two operational levels. Second, it does \emph{not} say that partially observed systems cannot be operated safely: safety may be obtained by design-time arguments that never consult the observations---if the certified reachable set excludes the ambiguous unsafe hypotheses, the premise is vacuous---or by admitting only actions viable at \emph{both} ambiguous hypotheses. What the theorem forecloses is the third option: resolving the ambiguity at runtime. Third, the bound prices the residual: partial distinguishability $\delta\in(0,1)$ buys a proportional relaxation, no more.

\section{Statistical Implementation}\label{sec:ledger}

Theorem~\ref{thm:separation}(b) consumes uniform one-step bounds $\abar_t$; this section produces them from statistical state estimates, with the bookkeeping that a sequential, adaptive, partially observed loop requires.

\subsection{Confidence-reachable tubes}

Two uncertainty sets must not be conflated. The \emph{information fiber} is the set-membership object: all states exactly compatible with the observation history and reachable from the certified initial set under the executed actions \cite{bertsekas1971set}; it is deterministic given pathwise noise bounds. Its statistical surrogate is a \emph{confidence-reachable tube}: an $\Hm_t$-measurable set-valued process $C_t$, intersected with the design-time reachable set (a deterministic pruning whose validity is part of the certificate, hence of $\eta$), for which one of the two guarantees below is available.

Three bookkeeping requirements, each commonly violated, are recorded first.

\emph{(i) Joint, not per-coordinate, coverage.} All requirements below concern the joint set $C_t$. Intersecting $d$ per-coordinate $95\%$ intervals is not a $95\%$ joint set: under independence the joint coverage is $0.95^{d}$ (already $0.51$ at $d=13$), and a generic dependence-agnostic repair is Bonferroni (per-coordinate $z\approx2.89$, not $1.96$, for a joint $95\%$ at $d=13$). Joint ellipsoidal or simultaneous sets should be used directly.

\emph{(ii) Time-uniformity.} Per-step coverage statements compose only by summation over the horizon. A \emph{confidence sequence} \cite{howard2021cs}, i.e.\ a tube with the time-uniform guarantee $\Prob(\forall t\le T:x_t\in C_t)\ge1-\Gamma$, pays a single horizon-wide $\Gamma$ and tolerates data-dependent monitoring and stopping. Time-uniform validity and per-step conditional calibration are \emph{distinct} properties; the horizon-uniform version of the ledger below deliberately requires only the former.

\emph{(iii) Uniformity over generators.} Because the conclusion of Theorem~\ref{thm:separation}(b) takes a supremum over $\Gc$, the tube guarantee must itself hold uniformly over $\Gc$. Set-membership tubes are uniform automatically: they are deterministic consequences of pathwise noise bounds and do not reference how actions were chosen. Statistical tubes are uniform when their validity argument is driven by the exogenous noise process only---self-normalized confidence sequences \cite{abbasi2011} provide time-uniform sets valid under adaptively chosen designs in linear-regression structure, and extending them to dynamically evolving latent states requires additional filtering assumptions that must be stated. Conformal constructions calibrated against the trajectory distribution of one nominal closed loop \cite{hu2026beliefsf} do not by themselves provide uniformity over $\Gc$.

\subsection{The risk ledger}

Given a tube, define the \emph{common admissible set}
\[
\Ac(C)\ :=\ \bigcap_{x\in C}\Av(x),
\]
the proposals viable at every state the tube cannot exclude.

\begin{proposition}[Sequential confidence-risk composition]\label{prop:ledger}
Suppose the intended admission condition is $\Ah_t\subseteq\Ac(C_t)$, and let $\varepsilon_1,\dots,\varepsilon_T\in[0,1]$ be deterministic scalars with
\[
\sup_{G\in\Gc}\ \esssup\ \Prob_G\bigl(\Ah_t\not\subseteq\Ac(C_t)\,\big|\,\Hm_t\bigr)\ \le\ \varepsilon_t .
\]

\emph{(i) Conditional version.} If, in addition, deterministic scalars $\gamma_t$ satisfy
$\sup_{G\in\Gc}\esssup\Prob_G(x_t\notin C_t\mid\Hm_t)\le\gamma_t$, then $\abar_t:=\gamma_t+\varepsilon_t$ are one-step setwise risk bounds, and Theorem~\ref{thm:separation}(b) yields
$\sup_{G}\Prob_G(\mathrm{violation})\le\sum_t(\gamma_t+\varepsilon_t)+\eta$.

\emph{(ii) Horizon-uniform version.} If instead the tube satisfies the generator-uniform time-uniform guarantee
\[
\sup_{G\in\Gc}\ \Prob_G\bigl(\exists\,t\le T:\ x_t\notin C_t\bigr)\ \le\ \Gamma ,
\]
then
\begin{equation}\label{eq:ledgerhorizon}
\sup_{G\in\Gc}\ \Prob_G(\mathrm{violation})\ \le\ \Gamma+\sum_{t=1}^{T}\varepsilon_t+\eta ,
\end{equation}
with no conditional-coverage claim required.
\end{proposition}

\begin{proof}
Pathwise, if $x_t\in C_t$ and $\Ah_t\subseteq\Ac(C_t)$ then $\Ah_t\subseteq\Av(x_t)$, since $x_t$ is among the states intersected over; hence
\[
B_t\ \subseteq\ \{x_t\notin C_t\}\ \cup\ \{\Ah_t\not\subseteq\Ac(C_t)\}
\]
with no probabilistic content. (i) follows by the conditional union bound and Definition~\ref{def:levels}. For (ii), the proof of Theorem~\ref{thm:separation}(b) gives, pathwise for every $G$,
\[
\{\mathrm{violation}\}\subseteq\Ccert^{\,c}\cup\{\exists t:x_t\notin C_t\}\cup\bigcup_t\{\Ah_t\not\subseteq\Ac(C_t)\},
\]
and the three probabilities are bounded by $\eta$, $\Gamma$, and $\sum_t\varepsilon_t$ respectively, each uniformly over $\Gc$.
\end{proof}

\begin{remark}[Where each risk is charged]\label{rem:charging}
The two risks are charged at different points of the decision cycle, and the conditioning follows the charging. The tube guarantee is a statement about the estimator before the admission computation runs, and $\varepsilon_t$ is a statement about that computation over its own randomization; both are therefore conditioned on $\Hm_t$. Conditioning $\varepsilon_t$ on $\Hf_t$ instead would be vacuous: $C_t$ is $\Hm_t$-measurable and $\Ah_t$ is $\Hf_t$-measurable, so $\{\Ah_t\not\subseteq\Ac(C_t)\}$ is $\Hf_t$-measurable and its conditional probability there takes only the values $0$ and $1$. The exactness statement of Theorem~\ref{thm:separation}(a), by contrast, is about the certified set as realized and is properly conditioned on $\Hf_t$.
\end{remark}

\begin{remark}[The deterministic-admission principle]\label{rem:deterministic}
If the admission computation uses no fresh randomization and enforces $\Ah_t\subseteq\Ac(C_t)$ by a deterministic, $\Hm_t$-measurable operation---containment checking, constraint tightening, projection---then the inclusion holds pathwise, $\varepsilon_t\equiv0$, and \eqref{eq:ledgerhorizon} collapses to
\[
\sup_{G\in\Gc}\Prob_G(\mathrm{violation})\ \le\ \Gamma+\eta .
\]
The requirement then ceases to be probabilistic at all: it becomes a design obligation discharged once, at construction time. The architectural reading deserves emphasis: \emph{the only statistical object in the loop should be the state confidence tube; the admission computation itself should be deterministic.} All statistical risk is then concentrated where estimation theory can control it, and the admission layer adds none. The ledger also explains why testing sampled actions statistically is dominated by certified-set computation: it makes $\varepsilon_t>0$ \emph{and} re-opens the selection channel of Section~\ref{sec:motivating}.
\end{remark}

\section{Consequences and Degenerate Cases}\label{sec:consequences}

\subsection{Monotonicity}

\begin{corollary}[Information monotonicity]\label{cor:mono}
If $C_1\subseteq C_2$ then $\Ac(C_1)\supseteq \Ac(C_2)$.
\end{corollary}

A tighter tube weakly enlarges the set of proposals that can be certified, and can therefore reduce conservatism; within this architecture, observation quality prices conservatism, not safety. The quantitative form---how tube geometry translates into admissible-set erosion and what excitation design can purchase---is developed in the companion paper \cite{companion}.

\subsection{The hard-safety trichotomy}

\begin{remark}[Operational taxonomy]\label{rem:trichotomy}
Under the hard worst-case admission semantics of this paper, the risk carried by observationally ambiguous hypotheses admits three dispositions: (i) \emph{design-time pruning}---the reachability argument excludes the ambiguous unsafe hypotheses from the tube, vacating the premise of Theorem~\ref{thm:impossibility}; (ii) \emph{common-feasible operation}---$\Ac(C_t)\neq\emptyset$, and the system runs on proposals viable at every unresolved hypothesis, without ever resolving the ambiguity; (iii) \emph{transfer to the certified fallback mode}---$\Ac(C_t)=\emptyset$, no proposal is certifiable, and the separately certified recovery mode is the only admissible output. Under the stated semantics these exhaust the set-based dispositions considered here; under chance-constrained or risk-budget semantics, where bounded risk may be accepted deliberately, the taxonomy is not exhaustive.
\end{remark}

\subsection{Degenerate cases}

\begin{remark}[Recovering the classics]\label{rem:classics}
\emph{Simplex} \cite{seto1998simplex}: single primary controller (one candidate per step, no retry), full state observation ($C_t=\{x_t\}$, $\gamma_t=0$), deterministic switching rule ($\varepsilon_t=0$); Proposition~\ref{prop:ledger}(ii) reduces to the classical guarantee $\Prob\le\eta$. \emph{ASIF / CBF-QP filtering} \cite{gurriet2018asif,ames2019cbf}: the filter's constraint set is a certified admissible set, so the architecture is selection-robust in the sense of Theorem~\ref{thm:separation}(a) under full-state, correct-model conditions; under partial observation it must be instantiated with measurement-robust variants \cite{lindemann2021robustocbf}, and Theorem~\ref{thm:impossibility} bounds what any instantiation can achieve. \emph{Shielding} \cite{alshiekh2018shielding}: a discrete-state instance of certified-set admission. The added generality of Theorem~\ref{thm:separation} over these classics is exactly the pair (arbitrary post-selecting generators, statistical partial observation)---the pair under which modern learned generators are actually operated.
\end{remark}

\section{Discussion}\label{sec:discussion}

\emph{What the theory licenses.} For learning-enabled loops the operational content of Theorem~\ref{thm:separation} is a change-management rule: the generator sits outside the safety case. The rule rests on the sufficiency half (Part (b)); the necessity half (Part (a)) is what forbids the alternative certification patterns of Section~\ref{sec:motivating}, and applies to admission soundness rather than to contract safety (Remark~\ref{rem:scope}). Retraining, fine-tuning, or wholesale replacement---including replacement by a generator that inspects the verifier and retries---requires no new safety argument, provided the certified triple (tube, admission computation, fallback) is untouched. Conversely, Example~\ref{ex:selection} is a warning about a certification pattern in current use: statistical per-candidate screening of sampled actions is not a gate, and an unadjusted per-candidate level does not remain a generator-level level under unrestricted selection.

\emph{Where the hard limits lie.} Theorem~\ref{thm:impossibility} places the unavoidable cost of partial observation in the admission layer: ambiguity between hypotheses that demand different decisions cannot be adjudicated at runtime, only priced ($\abar+\beta+\delta\ge1$), pruned at design time, or absorbed by common-feasible proposal sets. The practical consequence is a division of labor: estimation theory governs $\Gamma$, design-time verification governs $\eta$ and the pruning, and the admission layer, implemented deterministically, contributes nothing to the risk budget (Remark~\ref{rem:deterministic}).

\emph{Scope and companion work.} The results here are architectural; the single linear-Gaussian display \eqref{eq:gaussdelta} is illustrative. The quantitative questions this paper deliberately does not treat---how much observation information a given contract requires, the exact certifiability threshold for fixed excitation designs, the price of universal versus contract-specific design, and whether adaptive excitation can improve on the best fixed design---are the subject of a companion paper on the information geometry of admission \cite{companion}. The single-gate architecture is likewise a scope boundary rather than a technical convenience: Assumption~\ref{ass:nobypass} places one gate on the actuation path, and with several gates on distinct channels the information-set coupling identified in the decentralized supervisory-control setting \cite{barrett2000separation} reappears, so the uniform conditional bookkeeping of Definition~\ref{def:levels} would have to be redone. The present theory is also agnostic about how $\Av$ and the design-time certificates are produced; combining the ledger with a posteriori error estimates for the models used inside those certificates is a natural further step.


\end{document}